\documentclass[11pt]{article}

\pdfoutput=1

\usepackage[final, nonatbib]{nips_2016}
\usepackage{comment}
\usepackage{color}
\usepackage{url}
\usepackage{listings}
\usepackage{pgfgantt}
\usepackage{enumitem}
\usepackage{xspace}
\usepackage{amsfonts, amssymb, amsmath, amsthm}
\usepackage{algorithm}
\usepackage[colorinlistoftodos]{todonotes}

\newcommand{\eat}[1]{}

\newcommand{\y}{\ensuremath{\textbf{y}}\xspace}
\newcommand{\x}{\ensuremath{\textbf{x}}\xspace}

\newcommand{\X}{\ensuremath{\textbf{X}}\xspace}

\newcommand{\B}{\ensuremath{\textbf{B}}\xspace}

\newcommand{\N}{\ensuremath{\textbf{N}}\xspace}
\newcommand{\U}{\ensuremath{\textbf{U}}\xspace}

\newtheorem{defn}{Definition}
\newtheorem{lm}{Lemma}
\newtheorem{thm}{Theorem}
\newtheorem{col}{Corollary}

\usepackage{scalerel,stackengine}
\stackMath
\newcommand\reallywidehat[1]{%
	\savestack{\tmpbox}{\stretchto{%
			\scaleto{%
				\scalerel*[\widthof{\ensuremath{#1}}]{\kern-.6pt\bigwedge\kern-.6pt}%
				{\rule[-\textheight/2]{1ex}{\textheight}}
			}{\textheight}%
		}{0.5ex}}%
	\stackon[1pt]{#1}{\tmpbox}%
}
\renewcommand{\hat}[1]{\reallywidehat{#1}}
\newcommand{\ignore}[1]{}

\newif\iftr

\trtrue

\iftr
\newcommand{\apdx}[1]{Appendix~\ref{#1}\xspace}
\else
\newcommand{\apdx}[1]{the supplemental material\xspace}
\fi

\usepackage{enumitem}

\begin{document}

\title{Robust High-Dimensional Linear Regression}
\author{Chang Liu\\
	University of Maryland\\
	liuchang@cs.umd.edu\\
	 \And
	 Bo Li\\
	 Vanderbilt University\\
	 bo.li.2@vanderbilt.edu\\
	 \And
	 Yevgeniy Vorobeychik\\
	 Vanderbilt University\\
	 yevgeniy.vorobeychik@vanderbilt.edu\\
	 \And Alina Oprea \\
Northeastern University\\
a.oprea@ccs.neu.edu
	}
\date{}
\maketitle

\begin{abstract}
The effectiveness of supervised learning techniques has made them
ubiquitous in research and practice.
In high-dimensional settings, supervised learning commonly relies on dimensionality reduction to
improve performance and identify the most important factors in
predicting outcomes.
However, the economic importance of learning has made it a natural
target for adversarial manipulation of training data, which we term
\emph{poisoning attacks}.
Prior approaches to dealing with robust supervised learning rely on
strong assumptions about the nature of the feature matrix, such as
feature independence and sub-Gaussian noise with low variance.
We propose an integrated method for robust regression that relaxes
these assumptions, assuming only that the feature matrix can be well
approximated by a low-rank matrix.
Our techniques integrate improved robust low-rank matrix approximation
and robust principle component regression, and yield strong performance guarantees.
Moreover, we experimentally show that our methods significantly
outperform state of the art both in running time and prediction error.
\end{abstract}

\section{Introduction}

Machine learning has come to be widely deployed in a broad array of
applications.
An important class of applications of machine learning uses it to enable scalable security solutions, such as spam filtering, traffic analysis,
and fraud detection~\cite{androutsopoulos2000learning,chan1998toward,stolfo1997credit}. In these applications, reliability of the machine learning
system is crucial to ensure service quality and enforce security, but
strong incentives exist to reduce learning efficacy (e.g., to bypass
spam filters).
Indeed, recent research
demonstrates that existing systems are vulnerable in the presense of
adversaries who can manipulate either the training (i.e. the poisoning
attack) or test data (i.e. the evasion attack)
\cite{xiao2015feature,lowd2005adversarial,li2014feature,li2015scalable}.
Consequently, an important agenda in machine learning research is to
develop learning algorithms that are robust to data manipulation.
In this work,
we focus on designing supervised learning algorithms that are robust to poisoning attacks.

Existing research on robust machine learning dates back to algorithms
for robust PCA \cite{robustpca}.
Most of them assume that a portion of the underlying dataset is
randomly, rather than adversarially, corrupted.
Recently, Chen et al.~\cite{chen2013robust} and Feng et al.~\cite{feng2014robust} considered
recovery strategy when the corruption is adversarially chosen to achieve some malicious
goal.
The former considers a robust linear and the latter
logistic regression models.
However, both make an extremely strong assumption that each feature is
sub-Gaussian with vanishing variance (as $O(\frac{1}{n})$) and features are independent,
rendering them impractical and severely limiting the scope of associated theoretical guarantees.

In this work, we propose a novel algorithmic framework for making linear
regression robust to data poisoning.
Our framework does not require either sub-Gaussian or independence
assumptions on the feature matrix $\X$.
Instead, we assume that $\X$ is generated through adversarial corruption of an
approximately low-rank matrix.
Our goal is to make regression which uses dimensionality reduction,
such as PCA, robust to adversarial manipulation.
The technical challenge is two-fold: first, we must make sure that the
dimensionality reduction step can reliably recover the low-rank
subspace, and
second, that the resulting regression performed on the subspace can
recover sufficiently accurate predictions, in both cases despite both
noise and adversarial examples.
While these problems have previously been considered in isolation,
ours is the first integrated approach.
More significantly, the effectiveness of our approach relies on weaker
assumptions than prior art, and, as a result, our proposed practical algorithms
significantly outperform state-of-the-art alternatives.

Specifically, we assume that labels $\y$ are a linear function of the
true feature matrix $\X_\star$ with additive zero-mean noise.
In addition, $\X_\star$ is corrupted with noise, and the adversary
subsequently adds a collection of corrupted rows to the training data.
In this model, our approach involves two parts: first, we develop a
novel robust matrix factorization algorithm which correctly recovers
the subspace whenever this is possible, and second, a trimmed
principle component regression, which uses the recovered basis and
trimmed optimization to estimate linear model parameters.



Our main contributions are as follows:
\begin{itemize}[itemsep=0pt]
\item \textbf{Novel algorithm for robust matrix factorization: }
We develop a novel algorithm that reliably recovers the low-rank
subspace of the feature matrix despite both noise (about which we make
few assumptions) and adversarial examples.
We prove that
our algorithm is effective iff subspace recovery is possible.

\item \textbf{Novel robust regression algorithm with significantly
    weaker assumptions: }
In contrast to prior robust regression work, we do not require either
feature independence or low-variance sub-Gaussian distribution of
features.
We prove that our algorithm can reliably learn
the low-dimensional linear model despite data corruption and noise.



\item \textbf{Significant improvement in running time and accuracy: }
We present efficient algorithms which significantly outperform
prior art in running time and prediction efficacy.
\end{itemize}

\noindent{\bf Related Work: }
Robust PCA is widely used as a statistical tool for data analysis and dimensionality
reduction that is robust to i.i.d.\ noise \cite{robustpca}.
However, these methods cannot deal with ``malicious'' corruptions, where
the sophisticated adversaries can manipulate rows from the subspace of
the true feature matrix.
In contrast, our approach handles both noise and malicious corruption.
Recently, robust learning for several learning
models, such as linear and logistic regression have also been proposed \cite{chen2013robust,feng2014robust}.
The limitation of these approaches is their assumption that the
feature matrix is sub-Gaussian with vanishing variance, and that
features are independent.
Our approach, in contrast, only assumes that the true feature matrix
(prior to corruption) is low rank.
Yan \emph{et al.} proposed an outlier pursuit algorithm to deal with
the matrix completion problem with corruptions \cite{yan2013exact},
and a similar algorithm is applied by Xu \emph{et al.} to deal with the noisy version of feature matrix \cite{xu2010}. However, these methods only consider
the matrix recovery problem and are not scalable. A more scalable
algorithm based on the alternating minimization approach was recently
proposed by Rodriguez et al.~\cite{rodriguez2013fast}; however, this
method does not consider data noise or corruption.
A number of heuristic techniques have also been proposed for poisoning
attacks~\cite{barni2014source,rubinstein2009antidote,biggio2011bagging}
for other problems, such as robust anomaly detection source
identification.

\section{Problem Setup and Solution Overview}
\label{sec:prob}

We start with the pristine training dataset of $n$ labeled examples,
$\langle \X_\star, \y_\star\rangle$, which subsequently suffers from
two types of corruption: noise is added to feature vectors, and the
adversary adds $n_1$ malicious examples (feature vectors and labels)
to best mislead the learning.
We assume that the adversary has full knowledge of the learning algorithm.
The learner's goal is
to learn a model on the corrupted dataset which is similar to the true
model.
The feature space is high-dimensional, and the learner will perform
dimensionality reduction prior to learning.
In particular, we assume that $\X_\star$ is low-rank with a basis
$\B$, and we assume that the true model is the associated
low-dimensional linear regression.

Formally, observed training data is generated as follows:
\begin{enumerate}[itemsep=0pt]
       \item {\bf Ground truth:} $\y_\star = \X_\star \beta^\star =
         \U \beta^\star_U$, where $\beta^\star$ is the true model,
         $\beta^\star_U$ is its low-dimensional representation, and
         $\U = \X_\star B$ is the low-dimensional embedding of $\X_\star$.

        \item {\bf Noise: } $\X_0 = \X_\star + \N$, where $\N$ is a
          noise matrix with $\|\N\|_\infty \le \epsilon$; $\y_0 = \y_\star + e$, where $e$ is
          i.i.d.\ zero-mean
          Gaussian noise with variance $\sigma$.

          \item {\bf Corruption: } The attacker adds $n_1$
            adversarially crafted examples $(\x_{a}, \y_{a})$ to get
            $\langle\X, \y\rangle$, which maximally
            skews prediction performance of low-dimensional linear regression.

	

\end{enumerate}

To formally characterize how well the learner performs in this setting, we define
(1) a \emph{model function} $f(\X_0, \y_0)$ which is the
model learned on $\langle \X_\star, \y_\star\rangle$; (2) a
\emph{loss function} $l$; and (3) a \emph{threshold function} $\delta(z)$ which
takes as input $z>1$, and is increasing in $z$.
Our metric is \emph{$(f, l, \delta)$-tolerance}:
\begin{defn}[$(f, l, \delta)$-tolerance]
	We say that learner $\mathcal{L}$ is \emph{$(f, l, \delta)$-tolerant},
    if for any attacker, and any $z>1$, we have
	\(l(\mathcal{L}(\X, \y), f(\X_0, \y_0))\leq \delta(z)\)
	with probability at least $1-c_1 z^{-c_2}$, for some constant $c_1, c_2>0$.
\end{defn}
In our setting, $f(\X_0, \y_0)$ return $\beta^\star$ and $l$ is expected quadratic
loss $E_x\big[(x(\hat{\beta} - \beta^\star))^2\big]$.

For convenience, we let $\mathcal{O}$ denote the set of (unknown) indices of the
samples in $\X$ coming from $\X_0$ and
$\mathcal{A}=\{1,...,n+n_1\}-\mathcal{O}$ the set of indices
for adversarial samples in $\X$. For an index set $\mathcal{I}$ and matrix
$M$, $M^\mathcal{I}$ denotes the sub-matrix containing only rows in
$\mathcal{I}$; similar notation is used for vectors.
We define $\gamma=\frac{n_1}{n}$ as the
\emph{corruption ratio}, or the ratio of corrupted and pristine data.

\subsection{Solution overview and paper organization}

Our goal is to design a learner $\mathcal{L}$ to estimate the
coefficients $\hat{\beta}$ of the true model $\beta^\star$ using
low-dimensional embedding of a high-dimensional model.
We achieve this goal in two steps: (1) recover the subspace $\B$ of $\X_\star$; (2)
project $\X$ onto $\B$, and estimate $\hat{\beta}$ using
robust principle component regression. The key challenge is that an adversary can design
corrupted data to interfere both with the first and second step of the
process.

For the first step (Section~\ref{sec:rmf}), we develop a \emph{robust subspace recovery}
algorithm which can account for both noise $\N$ and adversarial
examples in correctly recovering the subspace of $\X_\star$.
We characterize necessary and sufficient conditions for successful
subspace recovery, showing that our algorithm succeeds whenever
recovery is possible.
The challenge in the second step (Section~\ref{sec:tpcr}) is that the adversary can construct $\X^\mathcal{A}$ from the
same subspace as $\X_\star$, but with the different distribution of
$\langle \X^\mathcal{A}, \y^\mathcal{A}\rangle$ from $\langle \X_\star, \y_\star\rangle$.
To address this, we propose the
\emph{trimmed principle component regression} algorithm to minimize
the loss function over only a subset of the dataset ensuring that the
adversary can have only a limited impact by adding $n_1$ arbitrary corrupted
samples without having these instances being discarded.
Our theoretical results demonstrate that the combined approach is
\emph{$(f, l, \delta)$-tolerant} learning algorithm.
Finally, in Section~\ref{sec:est}, we present an efficient practical
implementation of our methods, which we evaluate in Section~\ref{sec:exp}.


In our analysis, we use the corruption parameter $n_1$ and the rank $k$ of the 
low-dimensional embedding to characterize the theoretical results. In our experiments, 
however, we show that we only require a lower bound on $n_1$ and an upper bound on $k$ 
for our techniques to work.
\newcommand{\rank}[1]{{\ensuremath{\mathrm{rank}(#1)}}}
\newtheorem{problem}{Problem Definition}

\section{Robust Subspace Recovery}
\label{sec:rmf}

In this section, we discuss how to recover the low-rank subspace of
$\X_\star$ from $\X$. Our goal is to exactly recover the low-rank
subspace, i.e., returning a basis for $\X_\star$. We show sufficient
and necessary conditions for this problem to be solvable, and provide
algorithms when this is possible. As a warmup, we first discuss
the noise-free version of the problem, and then present our results
for the problem with noises.
Proofs of the theorems presented in this section can be found in
\apdx{proof:rmf}. Formally, we consider the following problem:
\begin{problem}[Subspace Recovery] Design an algorithm
$\mathcal{L}_\text{recovery}$, which takes as input $\X$,
and returns
a set of vectors $B$ which form the basis of $\X_\star$.
\label{prob:recover}
\end{problem}

\subsection{Warmup: Noise-free Subspace Recovery}

We first consider an easier version of Problem~\ref{prob:recover} with $N=0$.
In this case, we know that $\X^\mathcal{O}=\X_\star$.
We assume that we know $\rank{\X_\star}=k$ (or have an upper
bound on it).
Below we demonstrate that there exists a sharp threshold $\theta$ on $n_1$ such
that whenever $n_1 < \theta$, we can solve Problem~\ref{prob:recover}
exactly with high probability, whereas if $n_1 \ge \theta$,
Problem~\ref{prob:recover} cannot be solved.
To characterize this threshold, we define the cardinality of the \emph{maximal rank $k-1$
subspace} $\mathit{MS}_{k-1}(\X_\star)$ as the optimal value of the following
problem:

\[\max_{\mathcal{I}} |\mathcal{I}|
~~\text{s.t.}~~\rank{\X_\star^\mathcal{I}}\leq k-1\]

Intuitively, the adversary can insert $n_1=n-\mathit{MS}_{k-1}(\X_\star)$
samples to form a rank $k$ subspace, which does not span $\X_\star$. The
following theorem shows that in this case, there is indeed no learner
that can successfully recover the subspace of $\X_\star$.

\begin{thm} If $n_1+\mathit{MS}_{k-1}(\X_\star)\geq n$, then there
    exists an adversary such that no algorithm $\mathcal{L}_\text{recover}$
    solves Problem~\ref{prob:recover} with probability greater than $1/2$.
    \label{thm:3:1}
\end{thm}

On the other hand, when $n_1$ is below this threshold, we can use the
following simple algorithm to recover the subspace of $\X_\star$:

\begin{algorithm}
	\caption{Exact recover algorithm for Problem~\ref{prob:recover} (Noisy-free)}\label{alg:exact-pure}
	We search for a subset $\mathcal{I}$ of indices, such that $|\mathcal{I}|=n$,
    and $\rank{\X^\mathcal{I}}=k$\\
	\textbf{return} a basis of $\X^\mathcal{I}$.
\end{algorithm}
In fact, we can prove the following theorem.
\begin{thm} If $n_1+\mathit{MS}_{k-1}(\X_\star)<n$, then Algorithm
    ~\ref{alg:exact-pure} solves Problem~\ref{prob:recover} for any adversary.
    \label{thm:3:2}
\end{thm}

Theorems~\ref{thm:3:1} and~\ref{thm:3:2} together give the necessary and
sufficient conditions on when Problem~\ref{prob:recover} is solvable, and
Algorithm~\ref{alg:exact-pure} provides a solution.
We further show
an implication of these theorems on the corruption ratio $\gamma$.
We can prove that $\mathit{MS}_{k-1}(\X_\star)\geq k-1$
(see \apdx{proof:rmf}). Combining this with Theorem~\ref{thm:3:1}, we
can have the following upper bound on $\gamma$.
\begin{col} If $\gamma\geq 1-\frac{k-1}{n}$, then Problem~\ref{prob:recover}
    cannot be solved.
    \label{col:3:3}
\end{col}

\subsection{Dealing with Noise}

We now consider Problem~\ref{prob:recover} with noise. Before we discuss
the adversary, we first need to assume that the uncorrupted version is
solvable. In particular, we assume that $X_\star$ optimizes the following
problem:
\begin{subequations}
\begin{align}
	&\min_{\X'} || \X_0 - \X' ||\\
	&\text{s.t.}~~\rank{\X'}\leq k.
\end{align}
\end{subequations}
Without otherwise mentioned, we use $||\cdot||$ to denote the Frobenius norm.
We put no additional restrictions on $\N$ except above. Note that this
assumption is implied by the classical PCA problem~\cite{eckart1936approximation,hotelling1933analysis,jolliffe2002principal}. We want to
emphasize on the optimal value of the above problem. We denote this value
to be \emph{noise residual}, denoted as $\mathit{NR}(\X_0) = \N$. Noise residual is a key component to characterize the
necessary and sufficient condition for the solvability of Problem
~\ref{prob:recover}.

Characterization of the defender's ability to accurately recover the
true basis $\B$ of $\X_\star$ after the attacker adds $n_1$ malicious
instances stems from the attacker's ability to mislead the defender
into thinking that some other basis, $\bar{\B}$, better represents $\X_\star$.
Intuitively, since the defender does not know $\X_0$, $\X_\star$, or
which $n_1$ rows of the data matrix $\X$ are adversarial, this comes
down to the ability to identify the $n-n_1$ rows that correspond to
the correct basis (note that it will suffice to obtain the correct
basis even if some adversarial rows are used, since the adversary may
be forced to align malicious examples with the correct basis to evade
explicit detection).
As we show below, whether the defender can succeed is determined by
the relationship between the noise residual $\mathit{NR}(\X_0)$ and
\emph{sub-matrix residual}, denoted as $\mathit{SR}(\X_0)$, which we define as the value optimizing the following problem:
\begin{subequations}
\label{E:SR}
\begin{align}
&\min_{\mathcal{I}, \B, \U} || \X_0^\mathcal{I} - \U\bar{\B} ||\\
\text{s.t.} &\quad \rank{\bar{\B}}=k, \bar{\B}\bar{\B}^T=I_k, \X_\star \bar{\B}^T\bar{\B}\neq \X_\star\\
&\mathcal{I}\subseteq\{1,2,...,n\}, |\mathcal{I}|=n-n_1.
\end{align}
\end{subequations}

We now explain the above optimization problem. $U$ and $\bar{\B}$ are $n\times k$ and
$k\times m$ matrixes separately. Here $\bar{\B}$ is a basis which the
attacker ``targets''; for convenience,
we require $\bar{\B}$ to be orthogonal (i.e.,
$\bar{\B}\bar{\B}^T=I_k$).
Since the attacker succeeds only if they can induce a basis different
from the true $\B$,  we require that
$\bar{\B}$ does not span of $\X_\star$, which is equivalent to saying
$\X_\star \bar{\B}^T \bar{\B}\neq \X_\star$.
Thus, this optimization problem seeks
$n-n_1$ rows of $\X_\star$, where $\mathcal{I}$ is the corresponding
index set.
The objective is to minimize the distance between
$\X_0^\mathcal{I}$ and the span space of the target basis $\bar{\B}$,
(i.e., $||\X_0^\mathcal{I}-\U\bar{\B}||$).

\begin{algorithm}
	\caption{Exact recovery algorithm for Problem~\ref{prob:recover}}
    \label{alg:exact-noisy}
    Solve the following optimization problem and get $\mathcal{I}$.
	\begin{equation}
	{\begin{array}{c}
        \min_{\mathcal{I}, L} || \X^\mathcal{I} - L ||\\
	   \text{s.t.}~~\rank{L}\leq k, I\subseteq\{1, ..., n+n_1\}, |I|=n
    \end{array}}
    \label{eq:exact-noisy}
	\end{equation}
	\textbf{return} a basis of $\X^\mathcal{I}$.
\end{algorithm}

To understand the importance of $\mathit{SR}(\X_0)$, consider
Algorithm~\ref{alg:exact-noisy} for recovering the basis of
$\X_\star$, $\B$.
If the optimal objective value of optimization problem~\eqref{E:SR}, $\mathit{SR}(\X_0)$,
exceeds the noise $\mathit{NR}(\X_0)$, it follows that the defender
can obtain the correct basis $\B$ using Algorithm~\ref{alg:exact-noisy}, as it yields a better low-rank
approximation of $\X$ than any other basis.
Else, it is, indeed, possible for the adversary to induce an incorrect
choice of basis.
The following theorem formalizes this argument.

\begin{thm} If $\mathit{SR}(\X_0)\leq \mathit{NR}(\X_0)$,
	then no algorithm recovers the exact subspace of $\X_\star$ with
    probability greater than $1/2$. If $\mathit{SR}(\X_0)>\mathit{NR}(\X_0)$,
	then Algorithm~\ref{alg:exact-noisy} solves Problem~\ref{prob:recover}.
    \label{thm:3:3}
\end{thm}


To draw connection between the noisy case and the noise-free case, we can view
Theorem~\ref{thm:3:1} and~\ref{thm:3:2} as special cases of Theorem~~\ref{thm:3:3}.
\begin{thm} When $\N=0$, $\mathit{SR}(\X_0)>\mathit{NR}(\X_0) = 0$ if and
only if $n_1+\mathit{MS}_{k-1}(\X_\star)<n$.
\label{thm:3:4}
\end{thm}

\newcommand{\diag}[2]{\ensuremath{\text{diag}_{#1}\{#2\}}}

\section{Trimmed Principal Component Regression}
\label{sec:tpcr}

In this section, we present trimmed principal component regression (T-PCR)
algorithm. The key idea is to leverage the principal component regression
(PCR) approach to estimate $\hat{\beta}$, but during the process trimming
out those malicious samples that try to deviate the estimator from the true
ones. In the following, we present the approach, which is similar to the
standard PCR approach, though we do not require computing the exact
singular vectors of $\X_\star$.

Assume we recover a basis $\B$
of $\X_\star$. Without loss of generality, we
assume that $\B$ is an orthogonal basis of $k$ row vectors. Since $\B$ is a basis for
$\X_\star$, we assume $\X_\star=\U_\star\B$. Then we know that, by optimization (1),
$\U_\star=\text{argmin}_U||\X_0 - U\B||$. We compute $\U=\text{argmin}_U ||\X-U\B||$,
and, by definition, we know $\U_\star = \U^\mathcal{O}$. By OLS estimator, we know
that $\U^T=(\B\B^T)^{-1}\B\X^T$, and thus $\U=\X\B^T$.

To estimate $\y=\X_\star\beta+e$, we assume $\beta_U=\B\beta$.
Since $\X_\star=\U_\star\B$, we convert the estimation problem of $\beta$
from a high dimensional space to the estimation problem of $\beta_U$ from a
low dimensional space, such that
$y=\U\beta_U+e$. After estimating for $\hat{\beta_U}$, we can convert it back
to get $\hat{\beta}=\B\hat{\beta_U}$. Notice that this is similar to
principal component regression~\cite{jolliffe1982note}.

However, the adversary may corrupt $n_1$ rows
in $\U$ to fool the learner to make a wrong estimation on $\hat{\beta_U}$,
and thus on $\hat{\beta}$.
%
To mitigate this problem, we design Algorithm~\ref{alg:1}.
\begin{algorithm}[h]
	\caption{Trimmed Principal Component Regression}\label{alg:1}
    \textbf{Input:} $\X, \y$
    \begin{enumerate}[itemsep=0pt]
    \item Use Algorithm~\ref{alg:exact-noisy} to compute a basis from $\X$, and orthogonalize it to get $\B$
    \item Project $\X$ onto the span space of $\B$ and get $\U\leftarrow\X\B^T$
	\item Solve the following minimization problem to get $\hat{\beta_U}$
	\begin{equation}
		\min_{\beta_U} \sum_{j=1}^{n}\{(y_i - u_i\beta_U)^2
            \text{ for }i=1,...,n+n_1\}_{(j)}
		\label{eq:obj}
	\end{equation}
	where $z_{(j)}$ denotes the $j$-th smallest element in sequence $z$.
    \item \textbf{return} $\hat{\beta}\leftarrow\B\hat{\beta_U}$.
    \end{enumerate}
\end{algorithm}
Intuitively, during the training process, we trim out the top $n_1$
samples that maximize the difference between the observed response $y_i$ and
the predicted response $u_i\beta_U$, where $u_i$ denotes the $i$-th row of $U$.
Since we know the variances of these differences are small (i.e.,
recall Section~\ref{sec:prob}, $\sigma$ is the variance of the random noise
$y-x\beta^\star$), these samples corresponding to the largest
differences are more likely to be the adversarial ones.

Next, we theoretically bound the prediction differences between our model and
the linear regression model learnt on $\X_\star, \y_\star$.
\begin{lm}
	Algorithm~\ref{alg:1} returns $\hat{\beta}$, such that for any real value
	$h>1$ with at least $1-ch^{-2}$ probability for some constant $c$, we have
	\begin{equation}
	E_x\big[(x(\hat{\beta} - \beta^\star))^2\big]
        \leq 4\sigma^2\bigg(1+\sqrt{\frac{1}{1-\gamma}}\bigg)^2 \log{c}
	\label{eq:lm:1}
	\end{equation}
	\label{lm:1}
\end{lm}

We explain the intuition of this Lemma, and defers the detailed proof to
\apdx{proof:tpcr}. If an adversary wants to fool Algorithm~\ref{alg:1},
it needs to generate samples $(u_i, y_i)$, such that the loss function
$(y_i-u_i\hat{\beta_U})^2$ is among the smallest $n$. Since for samples
from $\X_\star, \y_\star$, there loss functions are already bounded
according to $\sigma$, the adversary does not have an ability
to significantly skew the estimator. In particular, if
$\sigma=0$, i.e.,
there is no error while generating $\y_0$ from $\X_\star$, then the
adversary can do nothing when $\gamma<1$, and thus the estimator is the same as the
linear regression's estimator on the uncorrupted data.

As an immediate consequence of Lemma~\ref{lm:1}, we have
\begin{thm} Given $\delta(c)=4\sigma^2\bigg(1+\sqrt{\frac{1}{1-\gamma}}\bigg)^2 \log{c}$,
Algorithm~\ref{alg:1} is $(f, l, \delta(c))$-tolerant.
\end{thm}

\section{Practical Algorithms}
\label{sec:est}

Algorithms~\ref{alg:exact-pure},~\ref{alg:exact-noisy},
and~\ref{alg:1} require enumerating a subset of indeces, and are thus all
exponential time. To make our approach practical, we develop
efficient implementations of Algorithms~\ref{alg:exact-noisy} and~\ref{alg:1}.

\subsection{Efficient Robust Subspace Recovery}

Consider the objective function~(\ref{eq:exact-noisy}). Since $\rank{L}\leq k$,
we can rewrite $L=\U\B^T$ where $\U$'s and $\B$'s shapes are $n\times k$, and
$m\times k$ respectively. Therefore, we can rewrite objective~
(\ref{eq:exact-noisy}) as
\[\min_{\mathcal{I}, \U, \B} ||\X^\mathcal{I}-\U\B^T||~~\text{s.t.}~~|\mathcal{I}|=n\]
which is equivalent to
\begin{equation}
\min_{\U, \B} \sum_{j=1}^{n}\{||x_i-u_i\B^T||\text{ for }i=1,...,n+n_1\}_{(j)}
\label{eq:rmf:main}
\end{equation}
where $x_i$ and $u_i$ denote the $i$th row of $\X$ and $\U$ respectively.
Solving Objective~\ref{eq:rmf:main} can be done using alternating minimization,
which iteratively optimizes the objective for $\U$ and $\B$ while fixing the other. Specifically, in the $w$th iteration, we optimize for the following
two objectives:
\[\U^{w+1} = \mathrm{argmin}_{U} ||\X-U(\B^w)^T||\]
\[\B^{w+1} = \mathrm{argmin}_{B} \sum_{j=1}^{n}\{||x_i-u_i^{w+1}B^T||\text{ for }i=1,...,n+n_1\}_{(j)}.\]
Notice that the second step computes the entire $\U$ regardless of the
sub-matrix restriction. This is because we need the entire $\U$ to be computed
to update $\B$. The key challenge is to compute $\B^{w+1}$ in each iteration,
which is, again, a trimmed optimization problem.
The next section presents a scalable solution for such problems.


\subsection{Efficient Algorithm for Trimmed Optimization Problems}

Both robust subspace recovery and optimizing for (\ref{eq:obj}) rely on solving
optimization problems in the form
\[\min_{\theta} \sum_{j=1}^n \{l(y_i, f_\theta(x_i))~\text{for}~i=1,...,n+n_1\}_{(j)}\]
where $f_\theta(x_i)$ computes the prediction over $x_i$ using parameter $\theta$,
and $l(\cdot, \cdot)$ is the loss function.
We refer to such problems as \emph{trimmed optimization problems}. It is easy to see
that solving this problem is equivalent to solving
$$
\begin{array}{c}
\min_{\theta, \tau_1,...,\tau_{n+n_1}} \sum_{i=1}^{n+n_1} \tau_il(y_i, f_\theta(x_i))\\
\text{s.t.}~~0\leq\tau_i\leq 1, \sum_{i=1}^{n+n_1}\tau_i=n
\end{array}
$$
We can use alternating minimization technique to solve this problem, by optimizing
for $\theta$, and $\tau_i$ respectively. We present this in Algorithm~\ref{alg:detail}.
In particular, the algorithm iteratively seeks optimal arguments for $\theta$
and $\tau_1,...,\tau_{n+n_1}$ respectively. Optimizing for $\theta$ is a
standard least square optimization problem. When optimizing
$\tau_1,...,\tau_{n+n_1}$, it is easy to see that $\tau_i=1$ if $l(y_i, f_\theta(x_i))$
is among the largest $n$; and $\tau_i=0$ otherwise. Therefore, optimizing
for $\tau_1,...,\tau_{n+n_1}$ is a simple sorting step.
While this algorithm is not guaranteed to converge to a global optimal, in our
evaluation,we observe that a random start of $\tau$ typically yields
near-optimal solutions.

\begin{algorithm}[t]
	\caption{Trimmed Optimization} 
    \begin{enumerate}[itemsep=0pt]
    \item Randomly assign $\tau_i\in\{0, 1\}$ for $i=1,...,n+n_1$, such that $\sum_{i=1}^{n+n_1}\tau_i = n$
    \item Optimize $\theta^{(j+1)}\leftarrow\mathrm{argmin}_{\theta}
    \sum_{i=1}^{n+n_1} \tau_i^{(j)}l(y_i, f_{\theta^{(j)}}(x_i))$;
    \item Compute $\mathrm{rank}_i$ as the rank of $l(y_i, f_{\theta^{(j+1)}}(x))$ in the ascending order;
    \item Set $\tau_i^{(j+1)}\leftarrow1$ for $\mathrm{rank}_i\leq n$, and $\tau_i^{(j+1)}\leftarrow0$ otherwise;
    \item Update $j\leftarrow j+1$ and go to 2 if there exists $i$ such that $\tau_i^{(j+1)}\neq \tau_i^{(j)}$;
    \item \textbf{return} $\theta^{(j)}$.
    \end{enumerate}
    \label{alg:detail}
\end{algorithm}

The following theorem shows that the algorithm converges.
\begin{thm} Given a loss function $l(y, f_\theta(x))$ such that a lower bound exists, i.e., 
	$\exists B. \forall y, \theta, x. l(y, f_\theta(x))\geq B$, Algorithm~\ref{alg:detail} converges. 
	That is, assuming $\ell_j=\sum_{i=1}^{n+n_1}\tau_i^{(j)}l(y_i, f_{\theta^{(j)}}(x_i))$ is the loss 
	after $j$-th iteration, then we have
\[\lim_{j\rightarrow+\infty} || l_{j+1}-l_j || = 0\]
\label{thm:conv}
\end{thm}
The proof can be found in Appendix~\ref{proof:conv}.

\ignore{
\subsection{Practical approach dealing with ranks}

Solving linear regress $y=Ax+\epsilon$ is problematic when $A^TA$ is not invertible, i.e., the optimal solution for 
$A$ is non-deterministic. This issue also causes the solution to the matrix factorization problem non-deterministic
especially when the parameter $k$ fed into the algorithm is larger than the matrix's intrinsic rank. Therefore, in 
practice, we typically require $||A||$ to be minimal while optimizing the loss. This is equivalent to adding a ridge
regularization in the loss function. Our experiments will show the effectiveness of this technique.
}

\section{Experimental Results}
\label{sec:exp}


We evaluate the proposed algorithms in comparison with
state-of-the-art alternatives on synthetic data.
For the subspace recovery problem, we compare to two state-of-the-art
approaches: Chen et al.~\cite{chen-icml11} and Xu et
al.~\cite{xu2010}.
We
compare the combined T-PCR algorithm with the recent robust regression
approach~\cite{chen2013robust}, and the standard ridge
regression algorithm.
For most experiments, we set $n+n_1=400$ and $m=400$. The only exception is
that, when we evaluate of the impact of $n+n_1$ on runtime, we vary
$n+n_1$ from 1,000 to 8,000.
We vary the rank $k$ of $\X_\star$, and $n_1$.
Results are averages of 30 runs after dropping the largest and smallest values.

\noindent{\bf Data: }
For a given
$(k, n_1)$, we generate $\X_\star$ as follows:
sample two matrixes $\U, \B$ with shape $n\times k$ and $k\times m$ respectively,
such that each element is sampled independently from $\mathcal{N}(0,
1)$, ensuring that both have rank $k$, and we set $\X_\star = \U\B$.
Next, we generate corruptions $\X^\mathcal{A}$ also as a low rank
matrix by generating $\U^\mathcal{A}$ and $\B^\mathcal{A}$ as above.
For $\B^\mathcal{A}$, we set the first half of
$\B^\mathcal{A}$ by choosing $k/2$ rows of $\X_\star$, and generating the
remaining $k/2$ rows randomly, ensuring that $\B$ has rank $k$.
We then concatenate $\X_\star$ and
$\X^\mathcal{A}=\U^\mathcal{A}\B^\mathcal{A}$ and shuffle
the rows.
We do not add noise to $\X_\star$, unless explicitly
stated. In doing so, we know that $\X^\mathcal{A}$ shares a common subspace
of rank $k/2$ with $\X_\star$, but the two subspaces are still different. This
approach of data generation is significantly more adversarial than
prior methods of generating adversarial instances, as we show below.
To generate labels, we generate a random $\beta^\star$, and $\y$ is
then constructed as described in Section~\ref{sec:prob}, where we
apply the method of Xiao et al.~\cite{xiao2015feature} to create
adversarial labels.
%


\noindent{\bf Runtime:} Figure~\ref{runtime}(a) and (b) presents the runtime comparison results.
Our approach is implemented in Scala without any special optimization,
and both~\cite{chen-icml11} and~\cite{xu2010} are
implemented in Matlab leveraging Matlab's built-in optimizations for matrix
operations.
In Figure~\ref{runtime}(a), we vary the rank from 1 to 20, and fix $n=350, n_1=50$.
Our algorithm is significantly faster than~\cite{chen-icml11}
and~\cite{xu2010}.
%
Figure~\ref{runtime}(b) shows runtime as a function of $n$, rank and
$n_1$ are fixed to 20 and 50 respectively.
We can observe that scalability of~\cite{chen-icml11}
and~\cite{xu2010} is quite poor in the size of the problem.
In contrast, our algorithm remains quite efficient (with running time
under $25$ seconds in all cases).
%


\noindent{\bf Identification rate of corrupted rows: }
Figure~\ref{runtime}(c)
presents the percentage of all corrupted rows identified by the algorithms.
We generate the data fixing intrinsic rank to be $10$, and varying inserted
corruptions from 10 to 150, keeping the total data size to be $400$.
We evaluate our approach varying algorithm parameter $k$ from 10 to 20, and parameter
$n$ from $210$ to $250$. The results show that our approach achieves 100\%
accuracy regardless of the chosen parameters as long as $k$ is no less than the intrinsic
rank, and $n$ is no bigger than the number of pristine rows.
We also compare our approach with prior work, \cite{xu2010} and 
\cite{chen-icml11}, which are identical. We refer to both as Xu et al.~\cite{xu2010}.
We can observe that the identification rate plummets for
$n_1\geq 20$, even though only 5\% of the rows are corrupted.

\begin{figure}[t]
\centering
\begin{tabular}{@{\hskip -0.5em}c@{\hskip -0.5em}c@{\hskip -0.5em}c@{\hskip -0.5em}}
\includegraphics[scale=0.24]{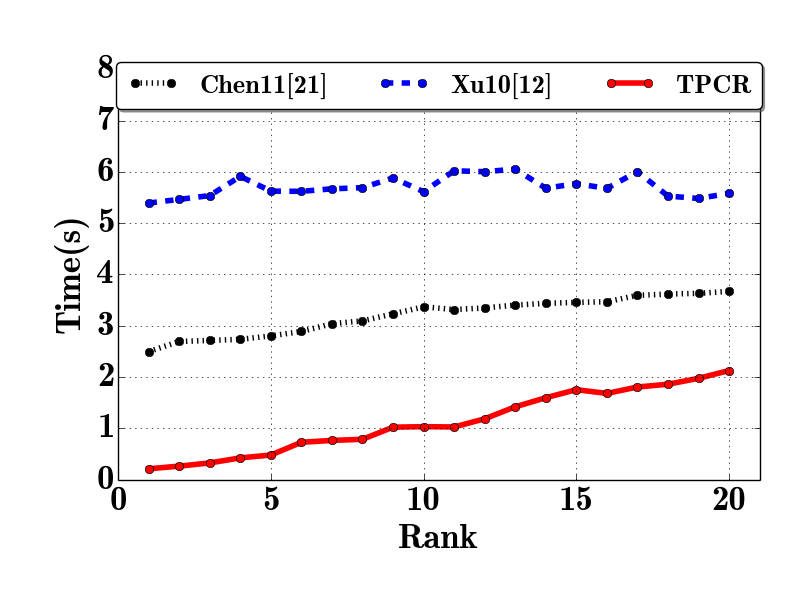}&
\includegraphics[scale=0.24]{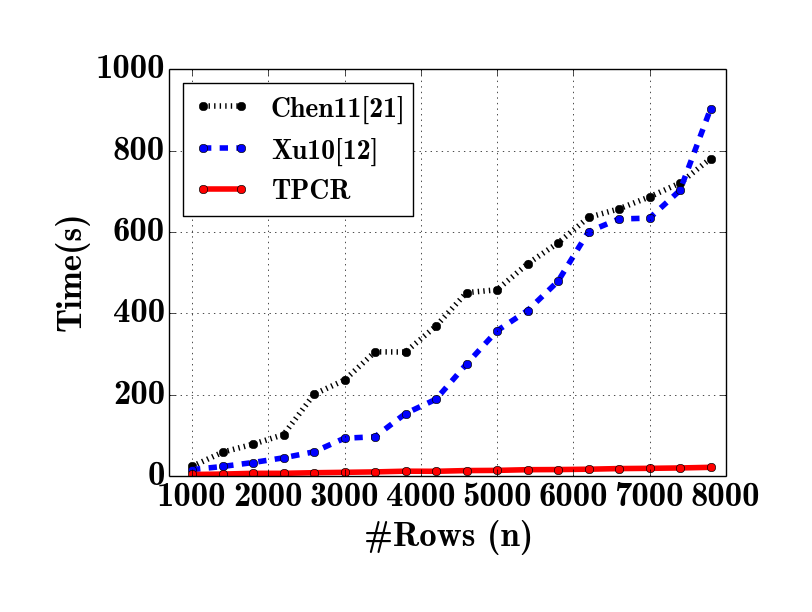}&
\includegraphics[scale=0.24]{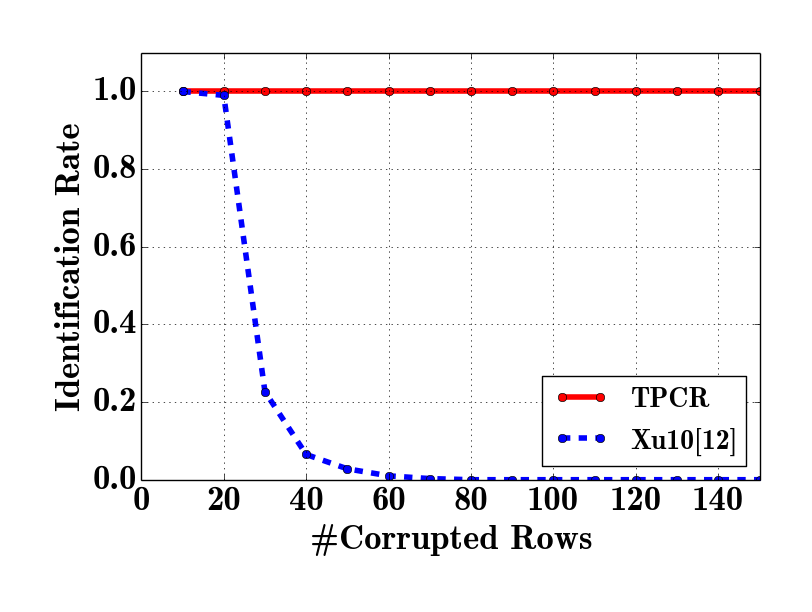}\\
(a) & (b) & (c)
\end{tabular}
\caption{(a) Runtime, as a function of rank. (b) Runtime, as a
  function of the number of rows (n). (c) Rate of correct identification
  of corrupted rows.\vspace{-1em}}
\label{runtime}
\end{figure}

\noindent{\bf Error on noisy data:} We add noise to $\X_\star$ to
evaluate performance on noisy data. Since~\cite{chen-icml11} cannot
handle noise, we only compare with~\cite{xu2010}. On each element
of $\X_\star$, we add a noise sampled from $\mathcal{N}(0, 0.01)$.
Figure~\ref{recovery}(a) and \ref{recovery}(b) show RMSE of the difference
from recovered $\X^\mathcal{O}$ and the true $\X_\star$. This metric is used
by~\cite{xu2010} as well. We use the grayscale to denote the RMSE:
lighter color corresponds to smaller RMSE.
On most test cases our algorithm successfully recovers the true subspace, while~\cite{xu2010}
fails on most cases.
Particularly, when $n_1<120$, our approach can completely
recover the underlying low-rank matrix.
When $n_1$ increases, the condition $\mathit{SR}(\X_0)>\mathit{NR}(\X_0)$
might not hold true, and Theorem~\ref{thm:3:3} says that no algorithm can
recover the true subspace with probability greater than $1/2$. However, this
theorem does not prevent our algorithm succeeding with probability
$<1/2$, which is why we observe several white spots for high $n_1$.

\begin{figure}[t]
\begin{tabular}{@{\hskip -0.5em}c@{\hskip -0.5em}c@{\hskip -0.5em}c}
\includegraphics[scale=0.24]{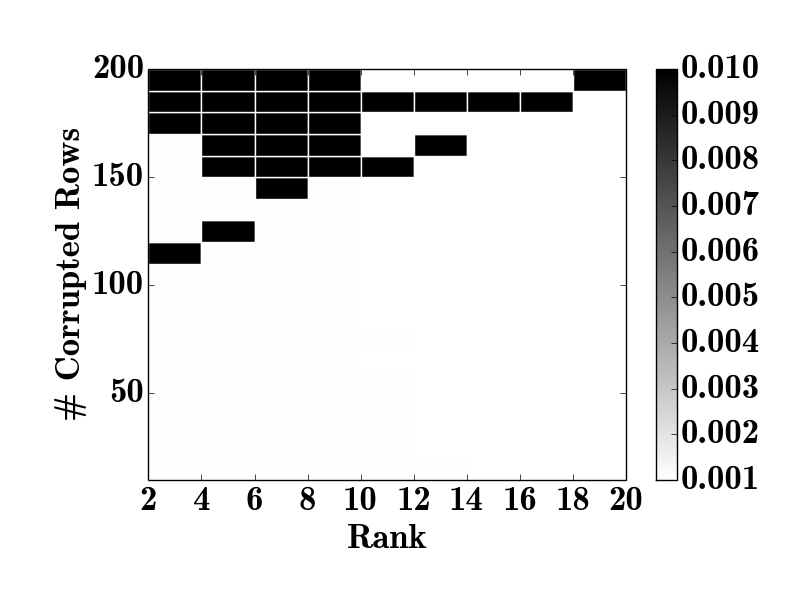} &
\includegraphics[scale=0.24]{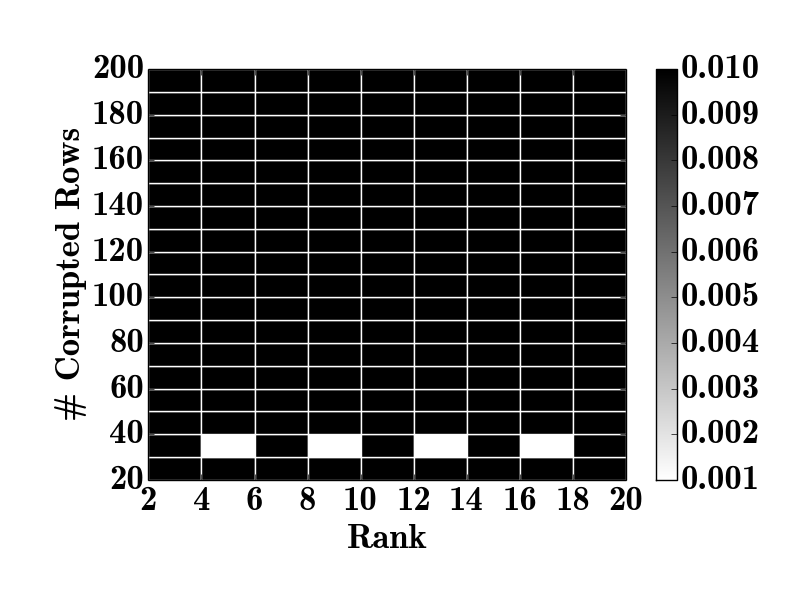} &
\includegraphics[scale=0.24]{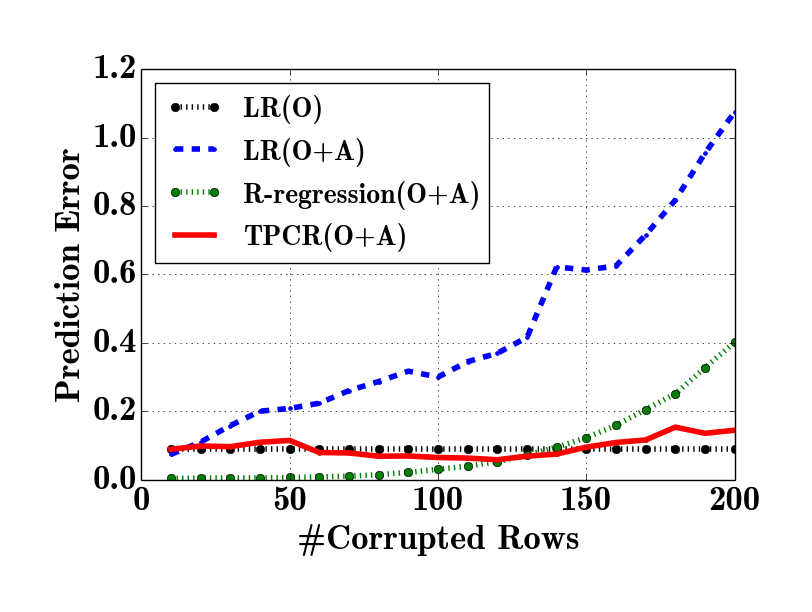}\\
(a) TPCR & (b) Xu et al.~\cite{xu2010} & (c)\\
\end{tabular}
\caption{(a) and (b) compares RMSE for T-PCR (a) and Xu et al.~\cite{xu2010} (b).
  (c) RMSE comparison with robust regression (Chen et al.
  \cite{chen2013robust}) and
  non-robust baseline.\vspace{-1em}}
\label{recovery}
\end{figure}


\noindent{\bf Robust Regression:} We evaluate our T-PCR
Algorithm~\ref{alg:1} against the robust regression method
of Chen et al.~\cite{chen2013robust}, which is the only alternative method for
linear regression robust to adversarial data poisoning.
As a baseline, we also present results for standard OLS linear
regression with and without adversarial instances (LR(O+A) and LR(O), respectively).
Results are evaluated using a ground truth test set not used for training.
The results, shown in Figure~\ref{recovery}(c), demonstrate that our
algorithm significantly outperforms the alternatives.
Indeed, while robust regression of Chen et al.~\cite{chen2013robust} does better than the
non-robust baseline, our method works nearly as well as linear
regression \emph{without} adversarial instances!
\section{Conclusion}
\vspace{-1em}

This paper considers the poisoning attack for linear regression problem with dimensionality reduction.
We address the problem in two steps: 1) introducing a novel robust matrix factorization method to recover the true subspace, and 2) novel robust principle component regression to prune adversarial instances based on the basis recovered in step (1).
We characterize necessary and sufficient conditions for our approach to be successful in recovering the true subspace, and present a bound on expected prediction loss compared to ground truth.
Experimental results suggest that the proposed approach is extremely effective, and significantly outperforms prior art.


\newpage
\appendix

\iftr

\section{Proof of Theorems in Section~\ref{sec:rmf}}
\label{proof:rmf}

Since, this section does not involve \y, we will omit \y without loss
of clarity.

\subsection{Theorem~\ref{thm:3:1}}
\begin{proof}[Proof of Theorem~\ref{thm:3:1}]
We prove by contradiction. Assume we have a learner
$\mathcal{L}_\text{recover}$, can solve Problem~\ref{prob:recover} with
probability more than 1/2. We want to show that there exist two different
spaces of rank-$k$, and one input $\X$ such that $\mathcal{L}_\text{recover}(\X_1)$
should return both two spaces with a probability $>1/2$,
which is impossible. In the following, we construct such two spaces.
Particularly, we will discuss how adversary can manipulate the matrix.

The adversary can choose
$\mathcal{I}$ which maximize $|\mathcal{I}|$ such that while
$\rank{\X_\star^\mathcal{I}}\leq k-1$. We know
$|\mathcal{I}|=\mathit{MS}_{k-1}(\X_\star)\geq n-n_1$.
This means that $|\mathcal{O}|-|\mathcal{I}|\leq n_1$.

Suppose $v_1,...,v_{k-1}$
be a set of basis for the row space of $\X_\star^\mathcal{I}$.
The adversary then choose a vector $v_k'$ which is orthogonal
to $\X_\star$. Then we know the span space of
$V'=\{v_1,...,v_{k_1}, v_k'\}$ is different from $\X_\star$.
Then the adversary draws $n_1$ samples from the span space of $V'$,
and insert them into $\X_\star$ to form $\X$.
Moreover, we denote $\X_\star'$ to be a matrix of $|\mathcal{I}|+n_1$
rows, so that the first $|\mathcal{I}|$ rows are $\X_\star^\mathcal{I}$,
and the rest $n_1$ rows sampled by the adversary. Therefore, we know
$\X_\star'$ is also a submatrix of $\X$, and we know that there
are at most $|\mathcal{O}|-|\mathcal{I}|\leq n_1$ rows in $\X$
not coming from $\X_\star'$.

In doing so, we know that $\X$ is constructed by corrupting $\X_\star$.
On the other hand, we can also see $\X$ as the result of corrupting
$\X_\star'$ by inserting $|\mathcal{O}|-|\mathcal{I}|\leq n_1$ rows.
Therefore, $\mathcal{L}_\text{recover}(\X_\star)$ should return
both $\X_\star$ and $\X_\star'$ with a probability greater than $1/2$,
which is impossible. Therefore, our conclusion holds true.
\end{proof}

\subsection{Theorem~\ref{thm:3:2}}
\begin{proof}[Proof of Theorem~\ref{thm:3:2}] We show that
Algorithm~\ref{alg:exact-pure} recovers the subspace of $\X_\star$
exactly. Assume $\B$ is returned by Algorithm~\ref{alg:exact-pure}
over $\X$. We only need to show that $\B$ is a basis of $\X_\star$.
By Algorithm~\ref{alg:exact-pure}, we know that $\B$ is
a basis of $n$ rows in $\mathcal{I}$ of $\X$. Since we know any
adversary can corrupt at most $n_1$ rows, thus
$|\mathcal{I}\cap\mathcal{A}|\leq n_1$. Therefore, by combining the
assumption $n_1+\mathit{MS}_{k-1}(\X_\star)<n$, we know that
\begin{equation}
|\mathcal{I}\cap\mathcal{O}|
=|\mathcal{I}|-|\mathcal{I}\cap\mathcal{A}|
\geq n-n_1
> \mathit{MS}_{k-1}(\X_\star)
\label{eq:lm:3:2:1}
\end{equation}
Therefore, we know $\B$ is a basis for $\X_\star^{\mathcal{I}\cap\mathcal{O}}$.
By the definition of $\mathit{MS}_{k-1}(\X_\star)$ and inequality~(\ref{eq:lm:3:2:1}),
we know that
\[\rank{\X_\star^{\mathcal{I}\cap\mathcal{O}}}=k\]
Therefore, we know that $\X_\star^{\mathcal{I}\cap\mathcal{O}}$ is exactly the
same subspace as $\X_\star$, and thus $\B$ is the basis of $\X_\star$.
\end{proof}

\subsection{Corollary~\ref{col:3:3}}
\begin{lm} $\mathit{MS}_{k-1}(\X_\star)\geq k-1$
\label{lm:app:1}
\end{lm}
\begin{proof} We can choose the $\mathcal{I}=\{1,...,k-1\}$, then we have
$\rank{\X_\star^{\mathcal{I}}}\leq k-1$. Therefore,
$\mathit{MS}_{k-1}(\X_\star)\geq |\mathcal{I}| = k-1$.
\end{proof}


Now, we can prove Corollary~\ref{col:3:3}.
\begin{proof}[Proof of Corollary~\ref{col:3:3}]
Given $\frac{n_1}{n}=\gamma\geq 1-\frac{k-1}{n}$, we have
\[n_1+(k-1)\geq n\]
Combining $\mathit{MS}_{k-1}(\X_\star)\geq k-1$,
we know
\[
n_1+\mathit{MS}_{k-1}(\X_\star)
\geq n_1+(k-1)
\geq n\]
Applying Theorem~\ref{thm:3:1}, we can conclude this corollary.
\end{proof}

\subsection{Theorem~\ref{thm:3:3}}
\begin{proof}[Proof of Theorem~\ref{thm:3:3}] The proof of this theorem is
similar to the proof of Theorem~\ref{thm:3:1} and~\ref{thm:3:2}. First,
when $\mathit{SR}(\X_0)\leq\mathit{NR}(\X_0)$, the adversary can construct
$\X$ such that two subspaces should be recovered with a probability greater
than $1/2$. Particularly, we assume $\mathcal{I}, \U, \B$ minimize objective
~\ref{E:SR}, and thus $\mathit{SR}(\X_0)=||\X_0^\mathcal{I}-\U\B||$.
The adversary samples $n_1$ rows $\X_\text{corrupt}$ from the span space
of $\B$, which does not belong to the span of $\X_\star$. We add a small noise
over $\X_\text{corrupt}$ to get $\X_1$, such that
(1) $\X_\text{corrupt}$ minimize $||\X_1-\X_\text{corrupt}||$;
and (2) $||\X_1-\X_\text{corrupt}||=\mathit{NR}(\X_0)-\mathit{SR}(\X_0)$.
Then the adversary insert $\X_1$ into $\X_0$ to get $\X$.
In this case, we know that $\X_\star$ optimizes its distance from $\X_0$,
while the $[\X_\star^\mathcal{I}; \X_\text{corrupt}]$ optimizes its distance
from $[\X_0^\mathcal{I};\X_\text{corrupt}]$, where we use $[A; B]$ to denote
the concatenation of rows from $A$ and $B$ respectively. Further, by definition,
we know both of these two distances is $\mathit{NR}(\X_0)$. Therefore,
the learner should recover from $\X$ both $\X_\star$ and
$[\X_0^\mathcal{I};\X_\text{corrupt}]$ with probability greater than $1/2$.
This is impossible! Therefore the first part of the theorem holds true.

For the second part, we follow the proof of Theorem~\ref{thm:3:2} verbatim,
and present the difference.
We show that Algorithm~\ref{alg:exact-noisy} recovers the subspace of
$\X_\star$ exactly. Assume $\B$ is returned by Algorithm~\ref{alg:exact-noisy}
over $\X$. We only need to show that $\B$ is a basis of $\X_\star$.
By Algorithm~\ref{alg:exact-noisy}, we know that $\B$ optimizes
its pan distance from a subset of $n$ rows in $\X$, which is denoted
as $\mathcal{I}$. Since we know any
adversary can corrupt at most $n_1$ rows, thus
$|\mathcal{I}\cap\mathcal{A}|\leq n_1$. Therefore, we know that
\begin{equation}
|\mathcal{I}\cap\mathcal{O}|
=|\mathcal{I}|-|\mathcal{I}\cap\mathcal{A}|
\geq n-n_1
\label{eq:lm:3:3:1}
\end{equation}
If $\B$ is not a basis of $\X_\star$, which means that $\X_\star\B^T\B\neq \X_\star$,
then we know that the distance between the span space of $\B$ and
$\X^{\mathcal{I}\cap\mathcal{O}}$ is greater than $\mathit{SR}(\X_0)>\mathit{NR}(\X_0)$.
This is impossible, since Algorithm~\ref{alg:exact-noisy} guarantees that
this distance should be no greater than $\mathit{NR}(\X_0)$. Contradiction!
Therefore the second part of the theorem holds true.
\end{proof}

\subsection{Theorem~\ref{thm:3:4}}
\begin{proof}[Proof of Theorem~\ref{thm:3:4}] When $\N=0$, we know that
$\mathit{SR}(\X_0)>\mathit{NR}(\X_0)$ if and only if $\mathit{SR}(\X_0)\neq 0$.
This means that for any $|\mathcal{I}|=n-n_1$, $\X_\star^\mathcal{I}=\U\B$
implies that $\X_\star\B^T\B=\X_\star$ (condition (\ref{E:SR}b)), which
implies that $\rank{\X_\star^\mathcal{I}|}=k$ for all $\mathcal{I}$.
Therefore, we know $\mathit{MS}_{k-1}(\X_\star)<n-n_1$, which
concludes this theorem.
\end{proof}
\newpage

\section{Proof of Lemma~\ref{lm:1}}
\label{proof:tpcr}

\begin{proof}
	Assume $\hat{\beta_U}$ is the solution for this optimization problem. 
	We assume the adversary wants to induce the regression system to compute $\hat{\beta_U}$.
    In this case, he has to corrupt $\gamma n$ rows in $U$. W.L.O.G. we can assume $\mathcal{O}=\{1,...,n_1\}$.
    We denote $\beta_U^\star=\B\beta^\star$. Since $\X_\star=\U_\star\B$, we know that
    \[\y-\X_\star\beta^\star=\y-\U_\star\beta_U^\star\]
	
	Since $\hat{\beta_U}$ optimize Eq (\ref{eq:obj}), we assume $(y_i-u_i\hat{\beta_U})^2$ are the smallest $n$ values for $i\in\{1,...,n\}$.

    Then we have
	
	\[\sum_{i=1}^{n_1} (y_i-u_i\hat{\beta_U})^2 + \sum_{i=n_1+1}^{n}(y_i-u_i\hat{\beta_U})^2 \leq \sum_{i=n_1+1}^{n+n_1} (y_i-u_i^T\beta_U^\star)^2\]
	
	Therefore we have
	
	\begin{equation}
		\sum_{i=n_1+1}^{n}(y_i-u_i\hat{\beta_U})^2 \leq \sum_{i=n_1+1}^{n+n_1} (y_i-u_i\beta_U^\star)^2
		\label{eq:2}
	\end{equation}
	
	Further, we know
	
	\begin{eqnarray}
		&&\sum_{i=n_1+1}^{n}(y_i-u_i\hat{\beta_U})^2\nonumber\\
		&=&\sum_{i=n_1+1}^{n}\bigg((y_i-u_i\beta_U^\star)+(u_i\beta_U^\star-u_i\hat{\beta_U})\bigg)^2 \nonumber\\
		&\geq&\sum_{i=n_1+1}^{n}\bigg\{ (y_i-u_i\beta_U^\star)^2+(u_i\beta_U^\star-u_i\hat{\beta_U})^2
                -2|y_i-u_i\beta_U^\star|\cdot|u_i\beta_U^\star-u_i\hat{\beta_U}| \bigg\} \nonumber\\
		&=& \sum_{i=n_1+1}^{n}(y_i-u_i\beta_U^\star)^2 + \sum_{i=n_1+1}^{n}(u_i(\beta_U^\star-\hat{\beta_U}))^2\nonumber\\
		&&\qquad\qquad\qquad- 2\bigg(\sum_{i=n_1+1}^{n}|u_i(\beta_U^\star-\hat{\beta_U})|\cdot |y_i-u_i\beta_U^\star|\bigg)
		\label{eq:3}
	\end{eqnarray}
	
According to Cauchy-Schwarz inequality, we have
    \[\bigg(\sum_{i=n_1+1}^{n}|u_i(\beta_U^\star-\hat{\beta_U})|\cdot |y_i-u_i\beta_U^\star|\bigg)^2
    \leq\bigg(\sum_{i=n_1+1}^{n}(u_i(\beta_U^\star-\hat{\beta_U}))^2\bigg)\cdot \bigg(\sum_{i=n_1+1}^n(y_i-u_i\beta_U^\star)^2\bigg)\]
    We assume $C=\sqrt{\sum_{i=n_1+1}^n (u_i(\beta^\star_U-\hat{\beta_U}))^2)}$, then, we have
    \begin{eqnarray}
    &&-2\bigg(\sum_{i=n_1+1}^{n}|u_i(\beta_U^\star-\hat{\beta_U})|\cdot |y_i-u_i\beta_U^\star|\bigg)\nonumber\\
    &\geq&-2\sqrt{\bigg(\sum_{i=n_1+1}^{n}(u_i(\beta_U^\star-\hat{\beta_U}))^2\bigg)\cdot \bigg(\sum_{i=n_1+1}^n(y_i-u_i\beta_U^\star)^2\bigg)}\nonumber\\
    &=&-2C\sqrt{\Sigma_{i=n_1+1}^n e_i^2}\nonumber
    \end{eqnarray}
    Substituting this inequality into (\ref{eq:3}) and combining with (\ref{eq:2}), we have
    \[\sum_{i=n_1+1}^{n+n_1}e_i^2\geq \sum_{i=n_1+1}^{n} e_i^2 + C^2 - 2C\sqrt{\Sigma_{i=n_1+1}^n e_i^2}\]
    By simple rearrangement, we have
	\[C^2 - 2C\sqrt{\sum_{i=n_1+1}^{n} e_i^2} \leq \sum_{i=n+1}^{n+n_1} e_i^2\]
	
	Since we know $y_i-u_i\beta_U^\star\sim\mathcal{N}(0, \sigma)$, we know that for any parameter $h>1$, we have
    $\text{Pr}(e_i\leq 2\sigma\sqrt{\log{h}})\geq 1-ch^{-2}$ for some constant $c$. Therefore, we know, with high probability (at least $1-ch^{-2}$), we have
	
	\begin{eqnarray}
		C^2 - 2\sqrt{n-n_1}C(2\sigma\sqrt{\log{h}})
        &\leq&C^2 - 2C\sqrt{\sum_{i=n+1}^{n} e_i^2} \nonumber\\
		&\leq&\sum_{i=n+1}^{n+n_1} e_i^2\nonumber\\
		&\leq&n_1(2\sigma\sqrt{\log{h}})^2\nonumber
	\end{eqnarray}
	
	Therefore, we have
	\[\bigg(C-2\sigma\sqrt{n-n_1}\sqrt{\log{h}}\bigg)^2 \leq n(2\sigma\sqrt{\log{h}})^2\]
	and thus
	\[C\leq 2\sigma\bigg(\sqrt{n}+\sqrt{n-n_1}\bigg) \sqrt{\log{h}}\]
    Therefore, we know
    \[\sqrt{\frac{\sum_{i=n_1}^n||u_i(\beta_U^\star - \hat{\beta_U})||^2}{n-n_1}}\leq 2\sigma\bigg(1+\sqrt{\frac{1}{1-\gamma}}\bigg) \sqrt{\log{h}}\]
    We notice the right hand side of the above inequality does not depend on $n, n_1$. Therefore, we take $n\rightarrow+\infty$,
    and we know that $n-n_1=(1-\gamma)n\rightarrow+\infty$,
    and apply the law of large numbers, we have
    \[\sqrt{E_u\big[(u(\beta_U^\star-\hat{\beta_U}))^2\big]}\leq 2\sigma\bigg(1+\sqrt{\frac{1}{1-\gamma}}\bigg) \sqrt{\log{h}}\]
    where left hand side is the same as $\sqrt{E_x\big[(x(\hat{\beta}-\beta^\star))^2]}$.
    Then the conclusion of Lemma~\ref{lm:1} is a simple rearrangement of the above inequality.
\end{proof}

\section{Proof of Theorem~\ref{thm:conv}}
\label{proof:conv}

We consider 
\[\ell_j = \sum_{i=1}^{n+n_1} \tau_i^{(j)} l(y_i, f_{\theta^{(j)}}(x_i))\]
and
\[\ell_j' = \sum_{i=1}^{n+n_1} \tau_i^{(j)} l(y_i, f_{\theta^{(j+1)}}(x_i))\]

According to the algorithm, it is easy to see that
\[\ell_j\geq\ell_j'\geq\ell_{j+1}\]

Therefore, $\ell_j$ is a monotonic decreasing sequence. Since a lower bound exists on the sequence, we assume
$\overline{B}=\inf\ell_j$ is the inferior of the sequence $\ell_j$. Therefore, we know that
\[\lim_{j\rightarrow+\infty} \ell_j-\overline{B} = \lim_{j\rightarrow+\infty} |\ell_j-\overline{B}| = 0\]

Further, we know that
\[0\leq |\ell_{j+1}-\ell_{j}| \leq |\ell_j-B|\]
and thus
\[0\leq \lim_{j\rightarrow+\infty}|\ell_{j+1}-\ell_{j}| \leq \lim_{j\rightarrow+\infty}|\ell_j-B| = 0\]

Therefore, we have
\[\lim_{j\rightarrow+\infty}|\ell_{j+1}-\ell_{j}|=0\]
Q.E.D.

\fi

\end{document}